\documentclass{article}

\usepackage[utf8]{inputenc} 
\usepackage[T1]{fontenc}    
\usepackage{hyperref}       
\usepackage{url}            
\usepackage{booktabs}       
\usepackage{amsfonts}       
\usepackage{nicefrac}       
\usepackage{microtype}      

\usepackage{xspace,amsmath,amssymb,amsthm,amsfonts,epsfig,syntonly,dsfont,pifont,natbib}
\usepackage{cite,bm,color,url,textcomp}
\usepackage{algorithmicx,algorithm}
\usepackage{epstopdf}
\usepackage{empheq,float}
\usepackage{graphicx,subfigure,balance}
\usepackage{multirow}
\usepackage{float}

\usepackage{times}
\usepackage[margin=1.08in]{geometry}

\newcommand{\mF}{\mathcal{F}}

\newcommand{\mR}{\mathcal{R}}
\newcommand{\R}{\mathbb{R}}

\newcommand{\E}{\mathbb{E}}
\newcommand{\indfunc}{\mathbb{I}}

\newcommand{\EE}[1]{\E\left[#1\right]}
\newcommand{\EEt}[1]{\E_t\left[#1\right]}
\newcommand{\iprod}[2]{\left\langle#1,#2\right\rangle}

\newcommand{\EEcct}[2]{\E_t\left[\left.#1\right|#2\right]}

\newtheorem{theorem}{Theorem}
\newtheorem{definition}{Definition}
\newtheorem{corollary}{Corollary}
\newtheorem{lemma}{Lemma}
\newtheorem{remark}{Remark}
\newtheorem{assu}{Assumption}

\newtheorem{claim}{Claim}

\title{\LARGE{\textbf{Adversarial Linear Contextual Bandits with Graph-Structured Side Observations}}}
\author{ 
	Lingda Wang\textsuperscript{1}, Bingcong Li\textsuperscript{2}, Huozhi Zhou\textsuperscript{1},\\ Georgios B. Giannakis\textsuperscript{2}, Lav R. Varshney\textsuperscript{1}, Zhizhen Zhao\textsuperscript{1} 	\vspace{0.1cm} \\
	\textsuperscript{1}{\normalsize\textit{ECE Department and CSL, University of Illinois at Urbana-Champaign, Urbana, IL 61801, USA}} \\	 
	\texttt{\normalsize\{lingdaw2, hzhou35, varshney, zhizhenz\}@illinois.edu} \\
	\textsuperscript{2}{\normalsize\textit{ECE Department, University of Minnesota - Twin Cities, Minneapolis, MN 55455, USA}} \\
	\texttt{\normalsize\{lixx5599, georgios\}@umn.edu}	
}

\date{\vspace{-3ex}}
\begin{document}
\maketitle
\begin{abstract}
This paper studies the adversarial graphical contextual bandits, a variant of adversarial multi-armed bandits that leverage two categories of the most common side information: \emph{contexts} and \emph{side observations}. In this setting, a learning agent repeatedly chooses from a set of $K$ actions after being presented with a $d$-dimensional context vector. The agent not only incurs and observes the loss of the chosen action, but also observes the losses of its neighboring actions in the observation structures, which are encoded as a series of feedback graphs. This setting models a variety of applications in social networks, where both contexts and graph-structured side observations are available. Two efficient algorithms are developed based on \texttt{EXP3}. Under mild conditions, our analysis shows that for undirected feedback graphs the first algorithm, \texttt{EXP3-LGC-U}, achieves the regret of order $\mathcal{O}(\sqrt{(K+\alpha(G)d)T\log{K}})$ over the time horizon $T$, where $\alpha(G)$ is the average \emph{independence number} of the feedback graphs. A slightly weaker result is presented for the directed graph setting as well. The second algorithm, \texttt{EXP3-LGC-IX}, is developed for a special class of problems, for which the regret is reduced to $\mathcal{O}(\sqrt{\alpha(G)dT\log{K}\log(KT)})$ for both directed as well as undirected feedback graphs. Numerical tests corroborate the efficiency of proposed algorithms.
\end{abstract}

\section{Introduction}
\label{sec:intro}
Multi-armed bandits (MAB)~\citep{thompson1933likelihood,lai1985asymptotically,auer2002finite,auer2002nonstochastic} is an online learning model of paramount importance for sequential decision making. Yielding algorithms with both theoretical guarantees and convenient implementations such as \texttt{UCB1}~\citep{auer2002finite}, Thompson sampling~\citep{agrawal2012analysis,kaufmann2012thompson,thompson1933likelihood}, \texttt{EXP3}~\citep{auer2002nonstochastic}, and $\texttt{INF}$~\citep{audibert2009minimax}, MAB has been of interest in many real-world applications: clinical trials~\citep{thompson1933likelihood}, web advertising~\citep{jiang2015online}, web search~\citep{kveton2015cascading, wang2019nearly}, and cognitive radio~\citep{maghsudi2016multi}, to just name a few. While the classical MAB has received much attention, this model may not be delicate enough for applications, since it does not fully leverage the widely available side information. This has motivated studies on \emph{contextual bandits}~\citep{li2010contextual,chu2011contextual,abbasi2011improved,agarwal2014taming} and \emph{graphical bandits}~\citep{mannor2011bandits,alon2013bandits,alon2015online,alon2017nonstochastic}, which aim to address two categories of the most common side information, \emph{contexts} and \emph{side observations}, respectively. In a contextual bandit problem, a learning agent chooses an action to play based on the context for the current time slot and the past interactions. In a graphical bandit setup, playing an action not only discloses its own loss, but also the losses of its neighboring actions. Applications of contextual bandits include mobile health~\citep{tewari2017ads} and online personalized recommendation~\citep{li2010contextual,li2019bandit,zhou2020near}, whereas applications of graphical bandits include viral marketing, online pricing, and online recommendation in social networks~\citep{alon2017nonstochastic,liu2018analysis}. 

However, contextual or graphical bandits alone may still not capture many aspects of real-world applications in social networks efficiently. As a motivating example, consider the viral marketing over a social network, where a salesperson aims to investigate the popularity of a series of products~\citep{lobel2017customer}. At each time slot, the salesperson could offer a survey (context) of some product to a user together with a promotion. The salesperson also has a chance to survey the user's followers (side observations) in this social network, which can be realized by assuming that i) if the user would like to get the promotion, the user should finish the questionnaire and share it in the social network, and ii) if the followers would like to get the same promotion, they need to finish the same questionnaire shared by the user. 

This example demonstrates the importance of a new MAB model that accounts for both context and side observations. Thus, designing pertinently efficient algorithms with guarantees is valuable, which is also recognized in the recent work on stochastic graphical contextual bandits~\citep{singh2020contextual}. Mechanically combining existing algorithms for contextual bandits and graphical bandits leads to algorithms with better empirical performance, compared to algorithms designed solely for contextual or graphical bandits. This can be verified by the genie-aided argument that side observations provide more information beyond the original contextual bandit problem, and will therefore not result in worse performance, if used properly. Certain theoretical guarantees can be derived if we adopt the results of contextual bandits as the worst case in the analysis. However, one should keep in mind that the merits of this paper are not just in combining formulations and algorithms: we will show that simply combining existing algorithms will result in intractable steps in analysis, and will not yield efficient algorithms with \emph{meaningful} theoretical guarantees capturing the benefits of both contexts and side observations.

In this paper, we present the first study on adversarial linear contextual bandits with graph-structured side observations (or simply, graphical contextual bandits). Specifically, at each time slot $t$, the adversary chooses the loss vector for each action in a finite set of $K$ actions, and then a learning agent chooses from this $K$-action set after being presented with a $d$-dimensional context. After playing the chosen action, the agent not only incurs and observes the loss of the chosen action, but also observes losses of its neighboring action in the feedback graph $G_t$, where the losses are generated by the contexts and loss vectors under the linear payoff assumption~\citep{agrawal2012analysis}. The goal of the agent is to minimize the regret, defined as the gap between the losses incurred by the agent and that of some suitable benchmark policy. Under mild conditions, we develop two algorithms for this problem with theoretical guarantees: i) \texttt{EXP3-LGC-U}, inspired by \texttt{EXP3-SET}~\citep{alon2013bandits,alon2017nonstochastic} and \texttt{LinEXP3}~\citep{neu2020efficient}; ii) \texttt{EXP3-LGC-IX}, inspired by \texttt{EXP3-IX}~\citep{kocak2014efficient} and \texttt{LinEXP3}. The contributions of the present work can be summarized as follows:
\begin{itemize}
    \item We present and study a new bandit model, graphical contextual bandits, which jointly leverages two categories of the most common side information: contexts and side observations. This new model generalizes the original contextual bandits and graphical bandits, and turns out to be more delicate in describing real-world applications in social networks.
    \item Under mild assumptions, we propose the first algorithm, \texttt{EXP3-LGC-U}, for the general adversarial graphical contextual bandits. When the feedback graphs $\{G_t\}_{t=1}^T$ are undirected, we show that \texttt{EXP3-LGC-U} achieves a regret $\mathcal{O}(\sqrt{(K+\alpha(G)d)T\log{K}})$, where $\alpha(G)$ is the average \emph{independence number} of $\{G_t\}_{t=1}^T$. In the directed graph setting, we show a slightly weaker result with a regret $\mathcal{O}(\sqrt{(K+\alpha(G)d)T}\log(KdT))$. When losses are non-negative, we develop the second algorithm, \texttt{EXP3-LGC-IX}, whose regret upper bound is $\mathcal{O}(\sqrt{\alpha(G)dT\log{K}\log(KT)})$ for both directed and undirected graph settings. 
    \item In all regret upper bounds of our novel algorithms, the dependencies on $d$ and $T$ match exactly the best existing algorithm \texttt{RealLinEXP3}~\citep{neu2020efficient} for adversarial linear contextual bandits. Furthermore, the dependency on $K$ of our proposed algorithms improves over \texttt{RealLinEXP3} as $\alpha(G)\le K$ always holds, where the quantity $\alpha(G)$ matches the lower bound $\Omega(\sqrt{\alpha(G)T})$ for adversarial graphical bandits~\citep{mannor2011bandits}. This comparison indicates that our proposed algorithms capture the benefits of both contexts and side observations.
    \item Numerical tests reveal the merits of the proposed model and algorithms over the state-of-the-art approaches.
\end{itemize}

The remainder of this paper is organized as follows. A brief literature review is presented in Section~\ref{sec:rw}. Problem formulations and necessary assumptions for analysis are introduced in Section~\ref{sec:setting}. The \texttt{EXP3-LGC-U} and \texttt{EXP3-LGC-IX} together with their analyses, are detailed in Sections~\ref{sec:set} and~\ref{sec:ix}, respectively. Finally, we conclude the paper in Section~\ref{sec:con}. The proofs are deferred to the Appendix.

\textbf{Notation.}  We use $\|x\|_2$ to denote the Euclidean norm of vector $x$; $\iprod{x}{y}$ stands for the inner product of $x$ and $y$. We also define $\EEt{\cdot}=\EE{\cdot\big|\mF_{t-1}}$ as the expectation given the filtration $\mF_{t-1}$.

\section{Related work}
\label{sec:rw}
\textbf{Contextual bandits:} Our paper studies a variant of adversarial contextual bandits, where adversarial contextual bandits were first investigated in~\citet{rakhlin2016bistro,syrgkanis2016efficient,syrgkanis2016improved} for arbitrary class of policies without the linear payoff assumption. More relevant to our paper is~\citet{neu2020efficient} that studied adversarial linear contextual bandits. Another category of contextual bandits is named as contextual bandits. For stochastic linear contextual bandits, \citet{auer2002adaptive,chu2011contextual,li2010contextual,abbasi2011improved} provided \texttt{UCB}-based algorithms; \citet{agrawal2013thompson,abeille2017linear} designed and analyzed a generalization of Thompson sampling for the contextual setting. Stochastic contextual bandits in generalized linear models are studied in~\citet{valko2013finite,filippi2010parametric,calandriello2019gaussian}. Stochastic contextual bandits with arbitrary set of policies can be found in~\citet{langford2008epoch,dudik2011efficient,agarwal2014taming,foster2019model,foster2018practical,foster2020beyond}. A neural net framework for stochastic contextual bandits with theoretical guarantees is proposed in~\citet{zhou2020neural}. Other interesting works include non-stationary contextual bandits~\citep{luo2018efficient,chen2019new}, fairness in contextual bandits~\citep{joseph2016fairness,chen2020fair}, etc. We refer the readers to~\citet{zhou2015survey} for a survey on contextual bandits.

\textbf{Graphical bandits:} If the contexts are not considered, our model will degenerate to Graphical bandits, which consider the side observations upon classical MAB. Graphical bansits were first proposed under the adversarial setting~\citep{mannor2011bandits}. Performance for the model was then improved in a series of works~\citep{alon2013bandits,kocak2014efficient,alon2015online,alon2017nonstochastic}, with best performance matching the lower bound $\Omega(\sqrt{\alpha(G)T})$. Most existing algorithms for adversarial graphical bandits are based on the classical \texttt{EXP3}. Graphical bandits has also been considered in the stochastic setting: \citet{caron2012leveraging} proposed a variant of \texttt{UCB1}; \citet{buccapatnam2014stochastic} improved the previous result via $\epsilon$-greedy and \texttt{UCB} with a well-designed linear programming; \citet{cohen2016online} developed an elimination-based algorithm that achieved the optimal regret; Thompson-sampling-based algorithms were recently proposed in~\citet{liu2018information,liu2018analysis}. Other related works include graphical bandits with noisy observations~\citep{kocak2016online,wu2015online}, graphical bandits with switching costs~\citep{arora2019bandits,rangi2019online}, graphical bandits with small-loss bound~\citep{lee2020closer,lykouris2018small}, etc. We refer the readers to~\citet{valko2016bandits} for a survey on graphical bandits.

\textbf{Graphical contextual bandits:} Recently, \citet{singh2020contextual} studied a stochastic variant of our model. \texttt{UCB} and linear programming (LP) based algorithms were proposed. The \texttt{UCB} based algorithm achieves a regret $\mathcal{O}(K\log{T})$, whereas the LP based approach achieves a better regret $\mathcal{O}(\chi(G)\log{T})$ with $\chi(G)$ denoting the dominant number. 

\section{Preliminaries}
\label{sec:setting}

We consider an adversarial linear contextual bandit problem with graph-structured side observations between a \emph{learning agent} with a finite action set $V:=\{1,\ldots,K\}$ and its \emph{adversary}. At each time step $t=1,2,\ldots,T$, the interaction steps between the agent and its adversary are repeated, which are described as follows. At the beginning of time step $t$, the feedback graph $G_t(V,\mathcal{E}_t)$ and a loss vector $\theta_{i,t}\in\R^d$ for each action $i\in V$ are chosen by the adversary arbitrarily, where $G_t$ can be directed or undirected, $V$ is the node set (the same as the action set $V$), and $\mathcal{E}_t$ is the edge set. Note that $G_t$ and $\theta_{i,t}$ are \emph{not} disclosed to the agent at this time. After observing a context $X_t\in\R^d$, the agent chooses an action $I_t\in V$ to play based on $X_t$, the previous interaction history, and possibly some randomness in the policy, and incurs the loss $\ell_t(X_t,I_t)=\iprod{X_t}{\theta_{I_t,t}}$. Unlike the recently proposed adversarial linear contextual bandits~\citep{neu2020efficient}, where only the played action $I_t$ discloses its loss $\ell_t(X_t,I_t)$, here we assume all losses in a subset $S_{I_t,t}\subseteq V$ are disclosed after $I_t$ is played, where $S_{I_t}$ contains $I_t$ and its neighboring nodes in the feedback graph $G_t$. More formally, we have that $S_{i,t}:=\{j\in V\big|i\xrightarrow{t}j\in\mathcal{E}_t\text{ or }j=i\}$, where $i\xrightarrow{t}j$ indicates an edge from node $i$ to node $j$ in a directed graph or an edge between $i$ and $j$ in an undirected graph at time $t$. These observations except for that of action $I_t$ are called \emph{side observations} in graphical bandits~\citep{mannor2011bandits}. In addition, an \emph{oracle} provides extra observations for all $i\in S_{I_t}$ (see Assumption~\ref{assu:3} for details). Before proceeding to time step $t+1$, the adversary discloses $G_t$ to the agent. 
\begin{remark}
\normalfont The way the adversary discloses $G_t$ in this paper is called the \textbf{uninformed} setting, where $G_t$ is disclosed \textbf{after} the agent's decision making. Contrarily, a simpler setting from the agent's perspective is called the \textbf{informed} setting~\citep{alon2013bandits}, where $G_t$ is disclosed \textbf{before} the agent's decision making. The uninformed setting is the minimum requirement for our problem to capture the benefits of side observations~\citep[Theorem 1]{cohen2016online}.
\end{remark}
Furthermore, we have the following assumptions for the above interaction steps.

\begin{assu}[i.i.d. contexts]
\label{assu:2}
The context $X_t\in\R^d$ is drawn from a distribution $\mathcal{D}$ independently from the choice of loss vectors and other contexts, where $\mathcal{D}$ is known by the agent in advance .
\end{assu}
\begin{assu}[extra observation oracle]
\label{assu:3}
Assume at each time step $t$, there exists an \textbf{oracle} that draws a context $\tilde{X}_t\in\R^d$ from $\mathcal{D}$ independently from the choice of loss vectors and other contexts, and discloses $\tilde{X}_t$ together with the losses $\tilde{l}_t(\tilde{X}_t,i)=\iprod{\tilde{X}_t}{\theta_{i,t}}$ for all $i\in S_{I_t,t}$ to the agent.
\end{assu}
\begin{assu}[nonoblivious adversary]
\label{assu:1}
 The adversary can be \textbf{nonoblivious}, who is allowed to choose $G_t$ and $\theta_{i,t},\forall i\in V$ at time $t$ according to arbitrary functions of the interaction history $\mF_{t-1}$ before time step $t$. Here, $\mF_{t}:=\sigma(X_s,\tilde{X}_s,I_s, G_s, \{\ell_s(X_s,i)\}_{i\in S_s}, \{\tilde{\ell}_s(\tilde{X}_s,i)\}_{i\in S_s},\forall s\le t)$ is the filtration capturing the interaction history up to time step $t$.
\end{assu}

\begin{remark}
\normalfont Assumption~\ref{assu:2} is standard in the literature of adversarial contextual bandits~\citep{neu2020efficient,rakhlin2016bistro,syrgkanis2016efficient,syrgkanis2016improved}. In fact, it has been shown that if both the contexts and loss vectors are chosen by the adversary, no algorithm can achieve a sublinear regret~\citep{neu2020efficient,syrgkanis2016efficient}. The oracle in Assumption~\ref{assu:3} is mainly adopted from the proof perspective, and its role will be clear in the analysis. In real-world applications, this oracle can be realized. Consider the viral marketing problem for an example. After the user and her/his followers complete the questionnaire and get the offers, they will probably purchase the products and leave online reviews after they experience those products. Then, the extra observations can be provided by those reviews. Assumption~\ref{assu:1} indicates $\theta_{t,i}$ is a random vector with $\EEt{\theta_{i,t}}=\theta_{i,t}$, and a similar result holds for $G_t$. Note that a bandit problem with a nonoblivious adversary is harder than that with an oblivious adversary~\citep{bubeck2012regret,lattimore2020bandit} that chooses all loss vectors and feedback graphs before the start of the interactions. 
\end{remark}

The goal of the agent is to find a policy that minimizes its \emph{expected cumulative loss}. Equivalently, we can adopt the \emph{expected cumulative (pseudo) regret}, defined as the maximum gap between the expected cumulative loss incurred by the agent and that of a properly chosen policy set $\Pi$, 
\begin{align*}
     \mR_T=\max_{\pi_T\in\Pi}\E\left[\sum_{t=1}^T\iprod{X_t}{\theta_{I_t,t}-\theta_{\pi_T(X_t),t}}\right]=\max_{\pi_T\in\Pi}\EE{\sum_{t=1}^T\sum_{i\in V}\left(\pi_t^a(i|X_t)-\pi_T(i|X_t)\right)\iprod{X_t}{\theta_{i,t}}},
\end{align*}
where the expectation is taken over the randomness of the agent's policy and the contexts. It is widely acknowledged that competing with a policy that uniformly chooses the best action in each time step $t$ while incurring an $o(T)$ regret is hopeless in the adversarial setting~\citep{bubeck2012regret,lattimore2020bandit}. Thus, we adopt the fixed policy set $\Pi$ proposed for adversarial linear contextual bandits~\citep{neu2020efficient}, 
\begin{equation}
\label{eq:1}
    \Pi:=\{\pi_T\big|\text{all policies }\pi_T:\R^d\mapsto V\},
\end{equation}
where the decision given by $\pi_T\in\Pi$ only depends the current received context $X_t$. The best policy $\pi^*_T\in\Pi$ is the one that satisfies the following condition
\begin{equation*}
    \pi^*_T(i|x)=\indfunc\{i=\arg\min_{j\in V}\sum_{t=1}^T\iprod{x}{\E[\theta_{j,t}]}\},\,\forall x\in\R^d,
\end{equation*}
which can be derived from the regret definition as shown in~\citet{neu2020efficient}. 

Before presenting our algorithms, we will further introduce several common assumptions and definitions in linear contextual bandits and graphical bandits. We assume the context distribution $\mathcal{D}$ is supported on a bounded set with each $x\sim \mathcal{D}$ satisfying $\|x\|_2\le \sigma$ for some positive $\sigma$. Furthermore, we assume the covariance $\Sigma=\E[X_tX_t^\top]$ of $\mathcal{D}$ to be positive definite with its smallest eigenvalue being $\lambda_\text{min}>0$. As for the loss vector $\theta_{i,t}$, we assume that $\|\theta_{i,t}\|_2\le L$ for some positive $L$ for all $i,\, t$. Additionally, the loss $\ell_t(x,t)$ is bounded in $[-1,1]$: $|\ell_t(x,i)|\le 1$ for all $x\sim \mathcal{D}$, $i$, and $t$. We have the following graph-theoretic definition from~\citet{alon2013bandits,alon2017nonstochastic,liu2018analysis}.
\begin{definition}[Independence number]
The cardinality of the maximum independent set of a graph $G_t$ is defined as the \textbf{independence number} and denoted by $\alpha(G_t)$, where an independence set of $G_t=(V_t,\mathcal{E}_t)$ is any subset $V'_t\in V_t$ such that no two nodes $i,j\in V'_t$ are connected by an edge in $\mathcal{E}_t$. Note that $\alpha(G_t)\le K$ in general. Without ambiguity, we use $\alpha({G}):= \frac{1}{T}\sum_{t = 1}^T \alpha(G_t)$ to denote the average \emph{independence number} of the feedback graphs $\{G_t\}_{t=1}^T$ in remainder of this paper.
\end{definition}

\section{The \texttt{EXP3-LGC-U} algorithm}
\label{sec:set}

\begin{algorithm}[htb]
	\caption{\texttt{EXP3-LGC-U}}
	\label{alg:exp3-lgc-u}
	\textbf{Input}: Learning rate $\eta>0$, uniform exploration rate $\gamma \in (0,1)$, covariance $\Sigma $, and action set $V$.\\
	\textbf{For} $t = 1, \dots, T$, \textbf{do:} 
    \begin{enumerate}
		\item Feedback graph $G_t$ and loss vectors $\{\theta_{i,t}\}_{i\in V}$ are generated but not disclosed.
		\item Observe $X_t\sim\mathcal{D}$, and for all $i\in V$, set
		\begin{equation}
		\label{eq:w1}
		    w_t(X_t,i) = \exp\left(-\eta\sum_{s=1}^{t-1}\iprod{X_t}{\hat{\theta}_{i,s}}\right).
		\end{equation}
		\item Play action $I_t$ drawn according to distribution $\pi_t^a(X_t):=(\pi_t^a(1\big|X_t),\ldots,\pi_t^a(K\big|X_t))$, where
		\begin{equation}
		\label{eq:pi1}
		    \pi_t^a(i\big|X_t)=(1-\gamma)\frac{w_t(X_t,i)}{\sum_{j\in V}w_t(X_t,j)}+\frac{\gamma}{K}.
		\end{equation}
		\item Observe pairs $(i,\ell_t(X_t,i))$ for all $i\in S_{I_t,t}$, and disclose feedback graph $G_t$.
		\item Extra observation oracle: observe $\tilde{X}_t\sim\mathcal{D}$ and pairs $(i,\tilde{\ell}_t(\tilde{X}_t,i))$ for all $i\in S_{I_t,t}$.
		\item For each $i\in V$, estimate the loss vector $\theta_{i,t}$ as
		\begin{equation}
		\label{eq:est1}
		    \hat{\theta}_{i,t}=\frac{\indfunc\{i\in S_{I_t,t}\}}{q_t(i\big|X_t)}\Sigma^{-1}\tilde{X}_t\tilde{\ell}_{t}(\tilde{X}_t,i), 
		\end{equation}
		where $q_t(i\big|X_t) =\pi_t^a(i\big|X_t)+ \sum_{j:j\xrightarrow{t}i}\pi_t^a(j\big|X_t)$.
	\end{enumerate}
	\textbf{End For}
\end{algorithm}

In this section, we introduce our first simple yet efficient algorithm, \texttt{EXP3-LGC-U}, for both directed and undirected feedback graphs, which is the abbreviation for ``\textbf{EXP3} for \textbf{L}inear \textbf{G}raphical \textbf{C}ontextual bandits with \textbf{U}niform exploration''. Detailed steps of \texttt{EXP3-LGC-U} are presented in Algorithm~\ref{alg:exp3-lgc-u}. The upper bounds for the regret of \texttt{EXP3-LGC-U} are developed in Section~\ref{sec:ana_set}. We further discuss our theoretical findings on \texttt{EXP3-LGC-U} in Section~\ref{sec:d_set}. The proofs for the Claims, Theorems, and Corollaries in this section are deferred to Appendix~\ref{apen:1}.

The core of our algorithm, similar to many other algorithms for adversarial bandits, is designing an appropriate estimator of each loss vector and using those estimators to define a proper policy. Following the \texttt{EXP3}-based algorithms, we apply an exponentially weighted method and play an action $i$ with probability proportional to $\exp(-\eta\sum_{s=1}^{t-1}\langle X_t,\hat{\theta}_{i,s}\rangle)$ (see Eq.~\eqref{eq:w1}) at time step $t$, where $\eta$ is the learning rate. More precisely, a uniform exploration $\gamma$ is needed for the probability distribution of drawing action (see Eq.~\eqref{eq:pi1}). The uniform exploration is to control the variance of the loss vector estimators, which is a key step in our analysis. At this point, the key remaining question is how to design a reasonable estimator for each loss vector $\theta_{i,t}$. The answer can be found in Eq.~\eqref{eq:est1}, which takes advantage of both the original observations and the extra observations from the oracle. Similar to \texttt{EXP3-SET}, our algorithm uses importance sampling to construct the loss vector estimator $\hat{\theta}_{i,t}$ with controlled variance. The term $q_t(i|X_t)$ in the denominator in Eq.~\eqref{eq:est1} indicates the probability of observing the loss of action $i$ at time $t$, which is simply the sum of all $\pi_t^a(j|X_t)$ for all $j$ that is connected to $i$ at time $t$. The reason we use $\tilde{\ell}(\tilde{X}_t,i)$ and $\tilde{X}_t$ instead of $\ell(\tilde{X}_t,i)$ and $X_t$ in constructing loss vector estimator $\hat{\theta}_{i,t}$ can be partly interpreted in the following two claims.
\begin{claim}
\label{claim:1}
The estimator $\hat{\theta}_{i,t}$ of the loss vector $\theta_{i,t}$ in Eq.~\eqref{eq:est1} is an unbiased estimator given the interaction history $\mF_{t-1}$ and $X_t$, for each $i\in V$ and $t$, i.e., $\EEcct{\hat{\theta}_{i,t}}{X_t} = \theta_{i,t}$.
\end{claim}

It is straightforward to show that the estimator $\hat{\theta}_{i,t}$ in Eq.~\eqref{eq:est1} is unbiased w.r.t. $\EEt{\cdot}$ and $\EE{\cdot}$ by applying the law of total expectation. However, if we use $X_t$ and $\ell(X_t,i)$ to construct $\hat{\theta}_{i,t}$ in Eq.~\eqref{eq:est1}, it will only be unbiased w.r.t. $\EEt{\cdot}$ and $\EE{\cdot}$, but not $\EEcct{\cdot}{X_t}$. This observation turns out to be essential in our analysis, which leads to the following immediate result of Claim~\ref{claim:1}.
\begin{claim}
\label{claim:2}
Let $\pi_T:\R^d\mapsto V$ be any policy in $\Pi$ and $\hat{\theta}_{i,t}$ follows Eq.~\eqref{eq:est1}. Suppose $\pi_t^a$ is determined by $\mF_{t-1}$ and $X_t$, we have 
\begin{align}
\label{eq:C2}
\EE{\sum_{t=1}^T\sum_{i\in V}\left(\pi_t^a(i|X_t)-\pi_T(i|X_t)\right)\iprod{X_t}{\theta_{i,t}}}=\EE{\sum_{t=1}^T\sum_{i\in V}\left(\pi_t^a(i|X_t)-\pi_T(i|X_t)\right)\iprod{X_t}{\hat{\theta}_{i,t}}}.   
\end{align}
\end{claim}

\begin{remark}
\label{remark:3}
\normalfont 
The advantages and properties of Claim~\ref{claim:2} are summarized as following. i) By applying the policy produced by \textup{\texttt{EXP3-LGC-U}} and the best policy in the fixed policy set $\Pi$ in Eq.~\eqref{eq:1}, the term in the right hand side of Eq.~\eqref{eq:C2} is exactly the regret $\mR_T$ of \textup{\texttt{EXP3-LGC-U}}. Given this property, the known loss vector estimate $\hat{\theta}_{i,t}$, instead of the unknown true loss vector $\theta_{i,t}$, can be applied directly to our analysis of the regret. ii) Claim~\ref{claim:2} is not confined to \textup{\texttt{EXP3-LGC-U}} and can be applied to other loss vector estimators that adopt different construction methods and any other benchmark policy, as long as Claim~\ref{claim:1} is satisfied. iii) Based on  Claim~\ref{claim:2}, some techniques in proving classical \textup{\texttt{EXP3}} can be utilized in our analysis of the regret.
\end{remark}
\begin{remark}
\label{remark:4}
\normalfont Claim~\ref{claim:2} exhibits several differences between adversarial contextual bandits and classical adversarial MAB. First, the benchmark policy $\pi_T(\cdot|X_t)$ depends on the contexts in adversarial contextual bandits, while the benchmark policy is the best fixed action in hindsight in classical adversarial MAB. Second, consider the regret definition of classical adversarial MAB, $\mR_{T}^{\textup{MAB}} = \max_{j\in V}\EE{\sum_{t=1}^T(\sum_{i\in V}\pi_t^{a,\text{MAB}}(i)\ell_{i,t})-\ell_{j,t}}$, where $\pi_t^{a,\text{MAB}}(i)$ is the policy produced by an \texttt{EXP3}-based algorithm and $\ell_{i,t}$ is the loss for action $i$ at time step $t$. Since no context exists here, it is natural to design an estimator $\hat{\ell}_{i,t}$ of $\ell_{i,t}$ that is unbiased w.r.t. $\EEt{\cdot}$, and a similar result as Claim~\ref{claim:2} can be proved. However, with the contexts, if the loss vector estimators are only unbiased w.r.t. $\EEt{\cdot}$ rather than $\EEcct{\cdot}{X_t}$, Claim~\ref{claim:2} will not hold as shown in the proof of Claim~\ref{claim:2} in Appendix~\ref{pf:claim2}.  
\end{remark} 
Remarks~\ref{remark:3} and~\ref{remark:4} explain the need of adopting the extra observation oracle in \textup{\texttt{EXP3-LGC-U}} and the way the loss vector estimator $\hat{\theta}_{i,t}$ is constructed.

\subsection{Regret analysis for \texttt{EXP3-LGC-U}}
\label{sec:ana_set}
Our main theoretical justification for the performance of \texttt{EXP3-LGC-U} summarized in Theorem~\ref{thm:1}. 

\begin{theorem}
\label{thm:1}
For any positive $\eta\in(0,1)$, choosing  $\gamma= \eta K\sigma^2/\lambda_\text{min}$, the expected cumulative regret of \textup{\texttt{EXP3-LGC-U}} satisfies:
\begin{equation*}
    \mR_t\le \frac{\log{K}}{\eta}+ \frac{2\eta K\sigma^2}{\lambda_\text{min}} T+\eta d\sum_{t=1}^T\EE{Q_t},
\end{equation*}
where $Q_t=\alpha(G_t)$ if $G_t$ is undirected, and $Q_t= 4\alpha(G_t)\log(4K^2/(\alpha(G_t)\gamma))$ if $G_t$ is directed.
\end{theorem}

The proof of Theorem~\ref{thm:1} is mainly based on the following Lemma~\ref{lem:1}, which is established on Claim~\ref{claim:2}.
\begin{lemma}
\label{lem:1}
Supposing $\left|\eta\iprod{X_t}{\hat{\theta}_{i,t}}\right|\le 1$, the expected cumulative regret of \textup{\texttt{EXP3-LGC-U}} satisfies
\begin{align}
\label{eq:lem1}
 \mR_T\le \frac{\log{K}}{\eta}+2\gamma T+\eta\EE{\sum_{t=1}^T\sum_{i\in V}\pi_t^a(i|X_t)\iprod{X_t}{\hat{\theta}_{i,t}}^2}.
\end{align}
\end{lemma}
The proof of Lemma~\ref{lem:1} is detailed in Appendix~\ref{pf:lemma1}. The last term in the right side of Eq.~\eqref{eq:lem1} can be further bounded using graph-theoretic results in \citet[Lemma 10]{alon2017nonstochastic} and~\citet[Lemma 5]{alon2015online}, which are restated in Appendix~\ref{pf:thm1}. 
\begin{remark}
\label{remark:5}
\normalfont According to Eq.~\eqref{eq:oracle_exp} in the proof of Theorem~\ref{thm:1} in Appendix~\ref{pf:thm1}, if the extra observation oracle is not adopted, we will have a higher-order term $\EE{X_t^\top\Sigma^{-1}X_t X_t^\top\Sigma^{-1}X_t}$. In general, it is hard to specify the relationship between this term and the dimension of contexts $d$. This explains why we adopt the oracle in the algorithm.
\end{remark}

We have the following two corollaries based on Theorem~\ref{thm:1}, where the notations follow~\citet{alon2013bandits,alon2017nonstochastic}.
\begin{corollary}
\label{col:1}
For the undirected graph setting, if $\alpha(G_t)\le \alpha_t$ for $t=1,\ldots, T$, then setting $\eta = \sqrt{\frac{\log{K}}{2K\sigma^2T/\lambda_{\text{min}}+d\sum_{t=1}^T \alpha_t}}$ gives 
\begin{equation*}
 \mR_T= \mathcal{O} \left(\sqrt{\left(2K\sigma^2T/\lambda_{\text{min}}+d\sum_{t=1}^T \alpha_t \right)\log{K}}\right).   
\end{equation*}
\end{corollary}

\begin{corollary}
\label{col:2}
For the directed graph setting, if $\alpha(G_t)\le \alpha_t$ for $t=1,\ldots, T$, and supposing that $T$ is large enough so that $\log(1/\gamma)\ge 1$, then setting $\eta = (2K\sigma^2T/\lambda_{\text{min}}+4d\sum_{t=1}^T \alpha_t)^{-\frac{1}{2}}$ 
gives: 
\begin{equation*}
    \mR_T=\mathcal{O} \left(\sqrt{ 2K\sigma^2T/\lambda_{\text{min}}+4d\sum_{t=1}^T \alpha_t }\log(KdT) \right).
\end{equation*}
\end{corollary}

\subsection{Discussion}
\label{sec:d_set}
Corollaries~\ref{col:1} and~\ref{col:2} reveal that by properly choosing the learning rate $\eta$ and the uniform exploration rate $\gamma$, the regret of \texttt{EXP3-LGC-U} can be upper bounded by $\mathcal{O}(\sqrt{(K+\alpha(G)d)T\log{K}})$ in the undirected graph setting, and $\mathcal{O}(\sqrt{(K+\alpha(G)d)T}\log(KdT))$ in the directed graph setting. Compared with state-of-the-art algorithms for adversarial linear contextual bandits, \texttt{EXP3-LGC-U} has tighter regret upper bounds in the extreme case when the feedback graph $G_t$ is a fixed edgeless graph ($\alpha(G)=K$), as~\citet{neu2020efficient} shows $\mathcal{O}(5T^{2/3}(Kd\log{K})^{1/3})$ for \texttt{RobustLinEXP3} and $\mathcal{O}(4\sqrt{T}+\sqrt{dKT\log{K}}(3+\sqrt{\log{T}}))$ for \texttt{RealLinEXP3}. It is easily verified that the dependencies on $d$ and $T$ in the regrets of \texttt{EXP3-LGC-U} match with the best existing algorithm \texttt{RealLinEXP3}. Furthermore, the dependence on $K$ of \texttt{EXP3-LGC-U} is matching with the lower bound $\Omega(\sqrt{\alpha(G)T})$ for graphical bandits~\citep{mannor2011bandits}, which improves over that of \texttt{RealLinEXP3} in general cases. Moreover, our result is also better than algorithms designed for adversarial contextual bandits with arbitrary class of policies~\citep{rakhlin2016bistro,syrgkanis2016efficient,syrgkanis2016improved}, which are not capable of guaranteeing an $\mathcal{O}(\sqrt{T})$ regret.
 
In addition, \citet{neu2020efficient} is different from ours in the following respects: i) loss vector estimator construction, and ii) proof techniques. First, the estimator in~\citet{neu2020efficient} is only unbiased w.r.t. $\EEt{\cdot}$ rather than $\EEcct{\cdot}{X_t}$. Second, their proof is conducted on an auxiliary online learning problem for a fixed context $X_0$ with $K$ actions (See~\citet[Lemmas 3 and 4]{neu2020efficient} for details). 
\section{The \texttt{EXP3-LGC-IX} algorithm}
\label{sec:ix}

In this section, we present another efficient algorithm, \texttt{EXP3-LGC-IX}, for a special class of problems when the support of $\theta_{i,t}$ and $X_t$ is non-negative, and elements of $X_t$ are independent. The motivation for such a setting still comes from the viral marketing problem. Suppose the agent has a questionnaire (context) of some product, which contains true/false questions that are positively weighted. In this case, the answers of users (loss vectors) will be vectors that contain only $0/1$ entries. Under the linear payoff assumption, the loss is non-negative. \texttt{EXP3-LGC-IX}, which is the abbreviation for ``\textbf{EXP3} for \textbf{L}inear \textbf{G}raphical \textbf{C}ontextual bandits with \textbf{I}mplicit e\textbf{X}ploration'', has the same regret upper bound for both directed and undirected graph settings, as shown in Section~\ref{sec:ana_ix}. The proofs for the Claims, Theorems, and Corollaries in this section are deferred to Appendix~\ref{apen:2}.

\begin{algorithm}[htb]
	\caption{\texttt{EXP3-LGC-IX}}
	\label{alg:EXP3-LGC-IX}
	\textbf{Input:} Learning rate $\eta_t>0$, implicit exploration rate $\beta_t \in (0,1)$, covariance $\Sigma$, and action set $V$.\\
	\textbf{For} $t = 1, \dots, T$, \textbf{do:}
	\begin{enumerate}
		\item Feedback graph $G_t$ and loss vectors $\{\theta_{i,t}\}_{i\in V}$ are generated but not disclosed.
		\item Observe $X_t\sim\mathcal{D}$, and play action $I_t$ drawn according to distribution $\pi_t^a(X_t):=(\pi_t^a(1\big|X_t),\ldots,\pi_t^a(K\big|X_t))$ with
		\begin{equation}
		\label{eq:ix_pi}
		    \pi_t^a(i\big|X_t)=\frac{w_t(X_t,i)}{\sum_{j\in V}w_t(X_t,j)},
		\end{equation}
		where  $w_t(X_t,i) = \frac{1}{K}\exp\left(-\eta_t\sum_{s=1}^{t-1}\iprod{X_t}{\hat{\theta}_{i,s}}\right)$.
		\item Observe pairs $(i,\ell_t(X_t,i))$ for all $i\in S_{I_t,t}$, disclose feedback graph $G_t$.
		\item Extra observation oracle: observe $\tilde{X}_t\sim\mathcal{D}$ and pairs $(i,\tilde{\ell}_t(\tilde{X}_t,i))$ for all $i\in S_{I_t,t}$.
		\item For each $i\in V$, estimate the loss vector $\theta_{i,t}$ as
		\begin{equation}
		\label{eq:est2}
		    \hat{\theta}_{i,t}=\frac{\indfunc\{i\in S_{I_t,t}\}}{q_t(i\big|X_t)+\beta_t}\Sigma^{-1}\tilde{X}_t\tilde{\ell}_{t}(\tilde{X}_t,i),
		\end{equation}
		where $q_t(i\big|X_t) = \pi_t^a(i|X_t)+ \sum_{j:j\xrightarrow{t}i}\pi_t^a(j\big|X_t)$.
	\end{enumerate}
	\textbf{End For}
\end{algorithm}

Algorithm~\ref{alg:EXP3-LGC-IX} shows the detailed steps of \texttt{EXP3-LGC-IX}, which follows the method of classical \texttt{EXP3} and is similar to \texttt{EXP3-LGC-U}. The main differences between \texttt{EXP3-LGC-IX} and \texttt{EXP3-LGC-U} are as follows. First, no explicit uniform exploration mixes with the probability distribution of drawing action (see Eq.~\eqref{eq:ix_pi}). In this case, for \texttt{EXP3-LGC-U} without uniform exploration, only a worse regret upper bound that contains $mas(G)$ rather than $\alpha(G)$ can be proved in the directed graph setting, where $mas(G)$ is the average \emph{maximum acyclic subgraphs number} and $mas(G)\ge \alpha(G)$. This result could be obtained by simply removing the uniform exploration part in the proof of \texttt{EXP3-LGC-U} and substituting Lemma~\ref{lem:directed} with~\citet[Lemma 10]{alon2017nonstochastic}. Second, biased loss vector estimator is adopted (see Eq.~\eqref{eq:est2}). Similar to \texttt{EXP3-IX}, this biased estimator ensures that the loss estimator satisfies the following claim, which turns out to be essential for our analysis.
\begin{claim}
\label{claim:3}
The estimator $\hat{\theta}_{i,t}$ of the loss vector $\theta_{i,t}$ for each $i\in V$ and $t$ satisfies
\begin{align}
\EEcct{\sum_{i\in V}\pi_t^a(i|X_t)\iprod{X_t}{\hat{\theta}_{i,t}}}{X_t} =\sum_{i\in V} \pi_t^a(i|X_t)\iprod{X_t}{\theta_{i,t}}-\beta_t\sum_{i\in V}\frac{\pi_t^a(i|X_t)}{q_t(i|X_t)+\beta_t}\iprod{X_t}{\theta_{i,t}}.
\end{align}
\end{claim}

\begin{remark}
\normalfont Claim~\ref{claim:3} indicates the loss estimators in \textup{\texttt{EXP3-LGC-IX}} are optimistic. The bias incurred by \textup{\texttt{EXP3-LGC-IX}} can be directly controlled by the implicit exploration rate $\beta_t$. This kind of implicit exploration actually has similar effect in controlling the variance of the loss estimators as explicit exploration (e.g., uniform exploration), though the approach is different. Notice that Claim~\ref{claim:3} does not hold if there is no extra observation oracle (see the proof in Appendix~\ref{pf:claim3} for details), which further demonstrates the necessity of the oracle.
\end{remark}

\subsection{Regret analysis for \texttt{EXP3-LGC-IX}}
\label{sec:ana_ix}
The upper bound on the regret of \texttt{EXP3-LGC-IX} follows Theorem~\ref{thm:2}, where the proof of Theorem~\ref{thm:2} is deferred to Appendix~\ref{pf:thm2}. Notice that a similar higher-order term as that in Remark~\ref{remark:5} appear in the proof of Theorem~\ref{thm:2}, if the extra observation oracle is not adopted. 
\begin{theorem}
\label{thm:2}
Setting $\beta_t=\sqrt{\log{K}/(K+\sum_{s=1}^{t-1}Q_s)}$ and $\eta_t = \sqrt{\log{K}/(dK+d\sum_{s=1}^{t-1}Q_s)}$, the expected regret of \textup{\texttt{EXP3-LGC-IX}} satisfies:
\begin{equation}
    \mR_T  \le 2(1+\sqrt{d}) \EE{\sqrt{\left(K+\sum_{t=1}^TQ_t\right)\log{K}}},
\end{equation}
for both directed and undirected graph settings, where $Q_t=2\alpha(G_t)\log\left(1+\frac{\lceil K^2/\beta_t\rceil+K}{\alpha(G_t)}\right)+2$.
\end{theorem}
Based on Theorem~\ref{thm:2}, we have the following corollary. 
\begin{corollary}
\label{col:3}
Suppose $\alpha(G_t)\le \alpha_t$ for $t=1,\ldots T$, the regret of \textup{\texttt{EXP3-LGC-IX}} satisfies
\begin{equation*}
   \mR_T=\mathcal{O} \left (\sqrt{\sum_{t=1}^T\alpha_td\log{K}\log \left( KT \right)} \right), 
\end{equation*}
for both directed and undirected graph settings.
\end{corollary}
Corollary~\ref{col:3} reveals that by adopting the learning rate $\eta_t$ and the implicit exploration rate $\beta_t$ adaptively, the regret of \texttt{EXP3-LGC-IX} can be upper bounded by $\mathcal{O}(\sqrt{\alpha(G)dT\log{K}\log(KT)})$ for both directed and undirected graph settings. This result indicates that \texttt{EXP3-LGC-IX} captures the benefits of both contexts and side observations, as discussed in Section~\ref{sec:d_set}. The \texttt{EXP3-LGC-IX} algorithm cannot handle negative losses due to the following two reasons. First, if the losses are negative, Claim~\ref{claim:3} does not hold. Second, although we can flip the sign of $\beta_t$ according to the sign of the loss vector to guarantee the optimism of the loss estimator, the graph-theoretic result (e.g.,~\citet[Lemma 2]{kocak2014efficient}) cannot be applied as $\beta_t$ is required to be positive.

\section{Numerical Results}
\label{sec:nu_exp}

We conduct the numerical tests on synthetic data to demonstrate the efficiency of the novel \texttt{EXP3-LGC-U} and \texttt{EXP3-LGC-IX} algorithms.

 We consider a setting of $K=10$ actions, with $d=10$ dimensional contexts observed iteratively on a $T=10^5$ time horizon. Each coordinate of context $X_t$ (or $\tilde{X}_t$) is generated i.i.d. from the Bernoulli distribution with support $\{0,1/\sqrt{d}\}$ and $p=0.5$, where the covariance of $X_t$ is $I_d/(4d)$, and $I_d$ is the identity matrix of size $d\times d$. The loss vectors are generated with a sudden change. Specially, for $t\in[1,50000]$, each coordinate of $\theta_{i,t}$ are set to be $\theta_{i,t}(j) = 0.1i|\cos{t}|/\sqrt{d}$, whereas $\theta_{i,t}(j) = 0.05i|\sin{t}|/\sqrt{d}$ for the remaining time steps, for all $j =1,\ldots, d$. We consider the time-invariant and undirected feedback graph structure for the purpose of performance validation. As depicted in Figure~\ref{fig:GS}, the feedback graph consists of a 9-vertex complete graph and one isolated vertex, where the independence number $\alpha(G)=2$. 

\begin{figure}[!ht]
    \centering
    \includegraphics[width=0.4\columnwidth]{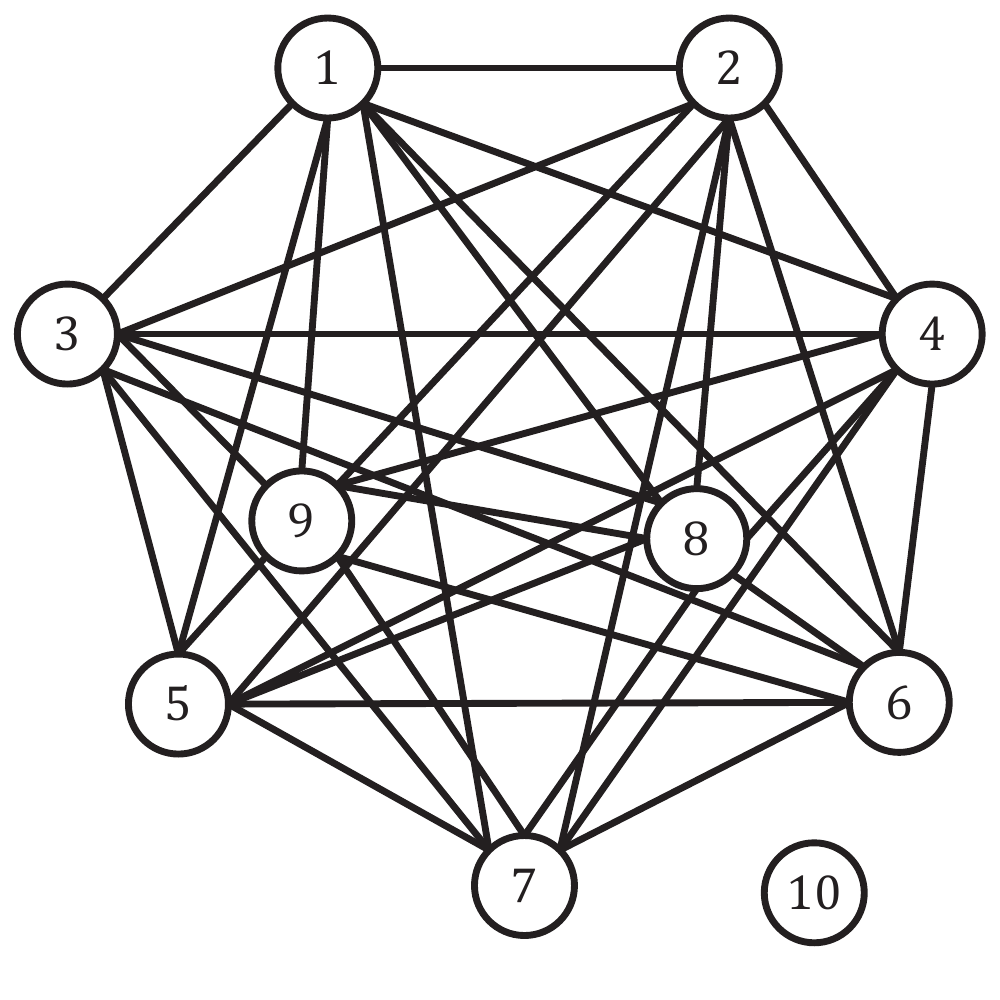}
    \caption{Feedback graph structure for the numerical tests.}
    \label{fig:GS}
\end{figure}

We compare \texttt{EXP3-LGC-U} and \texttt{EXP3-LGC-IX} with \texttt{RobustLinEXP3} from~\citet{neu2020efficient}. We also let \texttt{EXP3-LGC-U}$^*$ and \texttt{EXP3-LGC-IX}$^*$ denote the proposed algorithms without relying on side observations. The parameters of \texttt{EXP3-LGC-U} and \texttt{EXP3-LGC-IX} are chosen according to Corollary~\ref{col:1} and Theorem~\ref{thm:2}, respectively. For \texttt{EXP3-LGC-U}$^*$ and \texttt{EXP3-LGC-IX}$^*$, the parameter selection methods are identical as before, except for setting $\alpha(G) = K$. The parameters of \texttt{RobustLinEXP3} are tuned exactly the same as those in~\citet{neu2020efficient}.

\begin{figure}[!ht]
    \centering
    \includegraphics[width=0.6\columnwidth]{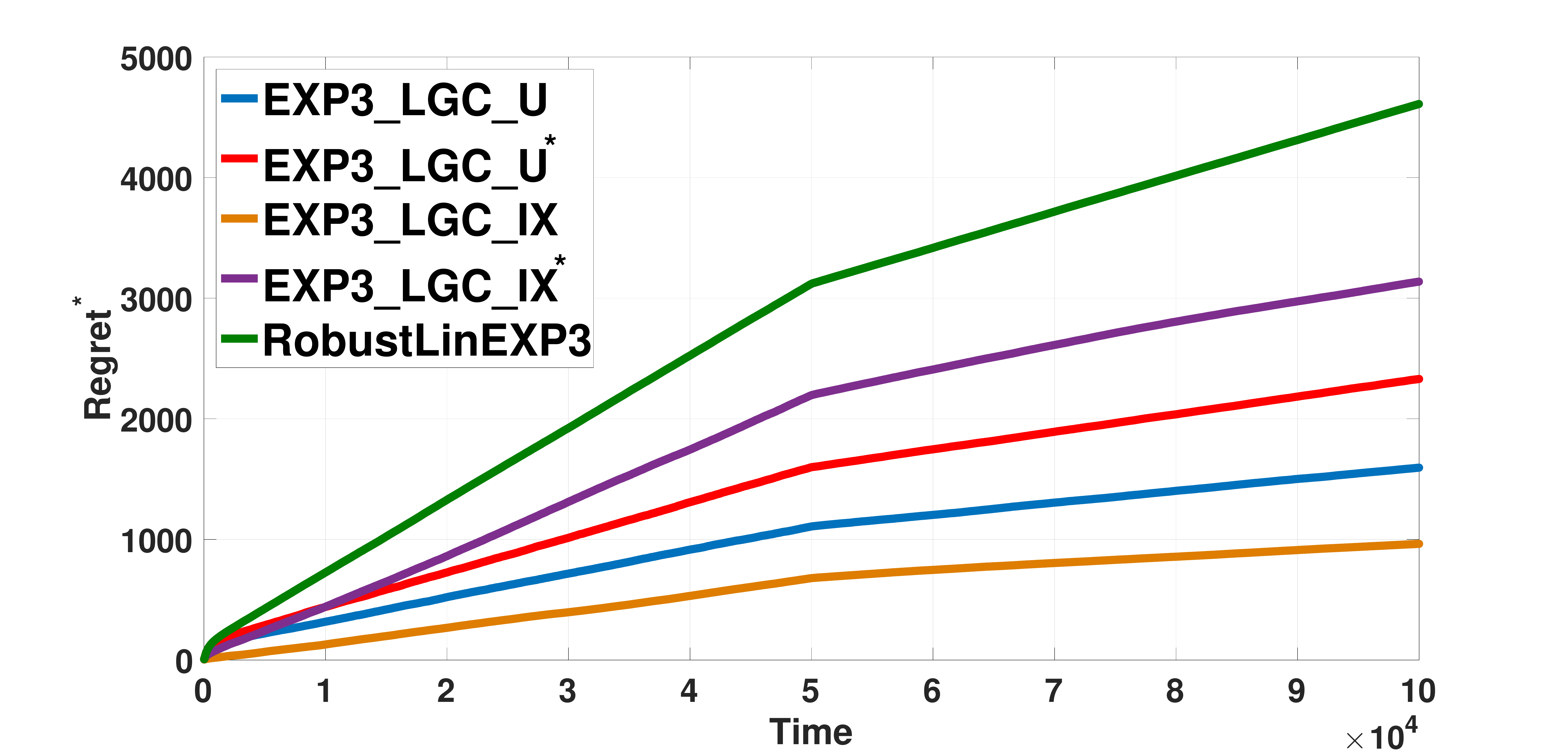}
    \caption{Regret$^1$ comparison of different algorithms on synthetic dataset over 100 independent trials.}
    \label{fig:reg}
\end{figure}

Figure~\ref{fig:reg} presents the expected cumulative regret\footnote{The regret before time step $T$ is actually the difference of accumulative losses between the algorithm and the benchmark policy for the horizon $T$, not the true regret defined for a horizon less than $T$.}, where the results are averaged over 100 independent trials. We find that \texttt{EXP3-LGC-U} and \texttt{EXP3-LGC-IX} significantly outperform the baseline algorithms (\texttt{RobustLinEXP3}, \texttt{EXP3-LGC-U}$^*$, and \texttt{EXP3-LGC-IX}$^*$), which is consistent with theoretical guarantees in Theorems~\ref{thm:1} and~\ref{thm:2}.
Besides, even if there is no side observation, our proposed algorithms are also better than $\texttt{RobustLinEXP3}$ (see comparison among \texttt{EXP3-LGC-U}$^*$, \texttt{EXP3-LGC-IX}$^*$, and \texttt{RobustLinEXP3} in Figure~\ref{fig:reg}). 

\section{Conclusion}
\label{sec:con}
We introduce a new MAB formulation -- adversarial graphical contextual bandits -- which leverage both contexts and side observations. Two efficient algorithms, \texttt{EXP3-LGC-U} and \texttt{EXP3-LGC-IX}, are proposed, with \texttt{EXP3-LGC-IX} for a special class of problems and \texttt{EXP3-LGC-U} for more general cases. Under mild assumptions, it is analytically demonstrated that the proposed algorithms achieve the regret $\widetilde{\mathcal{O}}(\sqrt{\alpha(G)dT})$ for both directed and undirected graph settings.

Several interesting questions are left open for future work. One challenging problem lies in providing a tight lower bound for adversarial linear graphical contextual bandits. Another promising direction for follow-up work is studying the small-loss bound for graphical contextual bandits.

\subsection*{Acknowledgements}
Lingda Wang and Zhizhen Zhao are supported in part by Alfred P. Sloan Foundation. Bingcong Li and Georgios B. Giannakis gratefully acknowledge
the support from NSF grants 1711471, and 1901134. Huozhi Zhou and Lav R. Varshney are funded in part by the IBM-Illinois Center for Cognitive Computing Systems Research (C3SR), a research collaboration as part of the IBM AI Horizons Network.
\bibliographystyle{apalike}
\bibliography{refs.bib}

\newpage
\onecolumn
\appendix
\begin{center}
\textbf{\Large Appendix}
\end{center}

\section{Proofs in Section~\ref{sec:set}}
\label{apen:1}

\subsection{Proof of Claim~\ref{claim:1}}
\label{pf:claim1}

\begin{proof}
By plugging Eq.~\eqref{eq:est1} into $\EEcct{\hat{\theta}_{i,t}}{X_t}$, we have that
\begin{align*}
\EEcct{\hat{\theta}_{i,t}}{X_t}&=\EEcct{\frac{\indfunc\{i\in S_{I_t}\}}{q_t(i\big|X_t)}\Sigma^{-1}\tilde{X}_t\tilde{\ell}_t(\tilde{X}_t,i)}{X_t}\\
&=\EEcct{\frac{\indfunc\{i\in S_{I_t}\}}{q_t(i\big|X_t)}\Sigma^{-1}\tilde{X}_t\tilde{X}_t^\top\theta_{i,t}}{X_t}\\
&\overset{(\text{a})}{=}\EEcct{\frac{\indfunc\{i\in S_{I_t}\}}{q_t(i\big|X_t)}}{X_t}\Sigma^{-1}\EE{\tilde{X}_t\tilde{X}_t^\top}\theta_{i,t},
\end{align*}
where step (a) uses the fact that $\EEcct{\theta_{i,t}}{X_t}=\EEt{\theta_{i,t}}= \theta_{i,t}$, and $\tilde{X}_t$ is independent of $\mF_{t-1}$, $X_t$, and $\theta_{i,t}$. Notice the following facts,
\begin{equation*}
    \EEcct{\frac{\indfunc\{i\in S_{I_t}\}}{q_t(i\big|X_t)}}{X_t} =\sum_{j:i\in S_{j,t}}\frac{\pi_t^a(j|X_t)}{q_t(i\big|X_t)}=1, \text{ and } \EE{\tilde{X}_t\tilde{X}_t^\top}=\Sigma.
\end{equation*}
We conclude that $\EEcct{\hat{\theta}_{i,t}}{X_t} = \theta_{i,t}$.
\end{proof}

\subsection{Proof of Claim~\ref{claim:2}}
\label{pf:claim2}
\begin{proof}
\begin{align*}
&\EE{\sum_{i\in V}\left(\pi_t^a(i|X_t)-\pi_T(i|X_t)\right)\iprod{X_t}{\hat{\theta}_{i,t}}}\\
&\quad\overset{(\text{a})}{=}
\EE{\EEcct{\sum_{i\in V}\left(\pi_t^a(i|X_t)-\pi_T(i|X_t)\right)\iprod{X_t}{\hat{\theta}_{i,t}}}{X_t}}\\&\quad = \EE{\sum_{i\in V}\left(\pi_t^a(i|X_t)-\pi_T(i|X_t)\right)\iprod{X_t}{\EEcct{\hat{\theta}_{i,t}}{X_t}}}\\
&\quad\overset{(\text{b})}{=} \EE{\sum_{i\in V}\left(\pi_t^a(i|X_t)-\pi_T(i|X_t)\right)\iprod{X_t}{\theta_{i,t}}},
\end{align*}
where step (a) uses the law of total expectation and step (b) uses Claim~\ref{claim:1}.
\end{proof}

\subsection{Proof of Lemma~\ref{lem:1}}
\label{pf:lemma1}
\begin{proof}
As mentioned before, Claim~\ref{claim:2} enables us to adopt the techniques similar to the ones used to originally analyze \texttt{EXP3} in~\citet{auer2002nonstochastic}. We introduce $W_t(x)=\sum_{i\in V}w_t(x,i)$ for convenience, where $w_t(x,i)$ is defined in Eq.~\eqref{eq:w1}. With the assumption that $|\eta\langle X_t,\hat{\theta}_{i,t}\rangle|\le 1$, the following result holds for each $t=1,\ldots, T$,
\begin{align}
\log{\frac{W_{t+1}(X_t)}{W_t(X_t)}}&= \log\left(\sum_{i\in V}\frac{w_{t+1}(X_t,i)}{W_t(X_t)}\right)\nonumber\\
&=\log\left(\sum_{i\in V}\frac{w_{t}(X_t,i)}{W_t(X_t)}\cdot e^{-\eta\iprod{X_t}{\hat{\theta}_{i,t}}}\right)\nonumber\\
&\overset{(\text{a})}{=}\log\left(\sum_{i\in V}\frac{\pi_t^a(i|X_t)-\gamma/K}{1-\gamma}\cdot e^{-\eta\iprod{X_t,\hat{\theta}_{i,t}}{}}\right)\nonumber\\
&\overset{(\text{b})}{\le}\log\left(\sum_{i\in V}\frac{\pi_t^a(i|X_t)-\gamma/K}{1-\gamma}\left(1-\eta\iprod{X_t}{\hat{\theta}_{i,t}}+\eta^2\iprod{X_t}{\hat{\theta}_{i,t}}^2\right)\right)\nonumber\\
&=\log\left(1+\sum_{i\in V}\frac{\pi_t^a(i|X_t)-\gamma/K}{1-\gamma}\left(-\eta\iprod{X_t}{\hat{\theta}_{i,t}}+\eta^2\iprod{X_t}{\hat{\theta}_{i,t}}^2\right)\right)\nonumber\\
&\overset{(\text{c})}{\le}\sum_{i\in V}\frac{\pi_t^a(i|X_t)}{1-\gamma}\left(-\eta\iprod{X_t}{\hat{\theta}_{i,t}}+\eta^2\iprod{X_t}{\hat{\theta}_{i,t}}^2\right)+\frac{\eta\gamma}{K(1-\gamma)}\sum_{i\in V}\iprod{X_t}{\hat{\theta}_{i,t}}\label{eq:lem1_1},
\end{align}
where equality (a) uses the definition of $\pi^a_{t}(i | X_t)$ in Eq.~\eqref{eq:pi1}, in step (a) the inequality $e^{-z}\le 1-z+z^2$ that holds for $z\ge -1$ is used, and in step (b) the inequality $\log(1+z)\le z$ that holds for $z> -1$ is used. 

The key to this proof is in the following. By drawing $X$ from the distribution $\mathcal{D}$ that is independent of the entire interaction history $\mF_T$, and substituting $X_t$ with $X$, we have that
\begin{equation*}
    \EE{\log{\frac{W_{t+1}(X_t)}{W_t(X_t)}}}= \EE{\log{\frac{W_{t+1}(X)}{W_t(X)}}}.
\end{equation*}
This is because $X_t$ and $X$ are i.i.d., and for each term $\log(W_{t+1}(X_t)/W_t(X_t))$, only $X_t$ is substituted with $X$ while $X_1,\ldots,X_{t-1}$ remain unchanged. Repeatedly, we apply this step to $\EE{\log{\frac{W_{t+1}(X_t)}{W_t(X_t)}}}$ for each $t$, which leads to the following lower bound,
\begin{align}
    \EE{\sum_{t=1}^T\log{\frac{W_{t+1}(X_t)}{W_t(X_t)}}}&=\EE{\sum_{t=1}^T\log{\frac{W_{t+1}(X)}{W_t(X)}}}\nonumber\\
    &= \EE{\log{\frac{W_{T+1}(X)}{W_1(X)}}}\nonumber\\
    &\overset{(\text{a})}{\ge}\EE{\log{\frac{w_{T+1}(X,\pi_T(X))}{W_1(X)}}}\nonumber\\
    & \overset{(\text{b})}{=} \EE{-\eta\sum_{t=1}^T\iprod{X}{\hat{\theta}_{\pi_T(X),t}}-\log{K}}\nonumber\\
    &\overset{(\text{c})}{=}\EE{-\eta\sum_{t=1}^T\iprod{X_t}{\hat{\theta}_{\pi_T(X_t),t}}-\log{K}}\nonumber\\
    & = \EE{-\eta\sum_{t=1}^T\sum_{i\in V}\pi_T(i|X_t)\iprod{X_t}{\hat{\theta}_{i,t}}-\log{K}},\label{eq:lem1_2}
\end{align}
where inequality (a) is due to the fact that $W_{T+1}(X)\ge w_{T+1}(X,\pi_T(X))$, step (b) is derived from the definition of $w_{T+1}(X,\pi_T(X))$ and the fact that $\log(W_1(X))=K$, and step (c) is realized by substituting $X$ with $X_t$ in each of $\iprod{X}{\hat{\theta}_{\pi_T(X),t}}$ as $X_t$ and $X$ are i.i.d. Combining the upper bound in Eq.~\eqref{eq:lem1_1} and the lower bound in Eq.~\eqref{eq:lem1_2} gives
\begin{align*}
&\EE{-\eta\sum_{t=1}^T\sum_{i\in V}\pi_T(i|X_t)\iprod{X_t}{\hat{\theta}_{i,t}}-\log{K}}\\
&\qquad\le\EE{\sum_{t=1}^T\sum_{i\in V}\frac{\pi_t^a(i|X_t)}{1-\gamma}\left(-\eta\iprod{X_t}{\hat{\theta}_{i,t}}+\eta^2\iprod{X_t}{\hat{\theta}_{i,t}}^2\right)+\frac{\eta\gamma}{K(1-\gamma)}\sum_{i\in V}\iprod{X_t}{\hat{\theta}_{i,t}}}.
\end{align*}
Reordering and multiplying both sides by $\frac{1-\gamma}{\eta}$ gives
\begin{align*}
    &\EE{\sum_{t=1}^T\sum_{i\in V}(\pi_t^a(i|X_t)-\pi_T(i|X_t))\iprod{X_t}{\hat{\theta}_{i,t}}}\\
    &\qquad\le\frac{(1-\gamma)\log{K}}{\eta}+\eta\EE{\sum_{t=1}^T\sum_{i\in V}\pi_t^a(i|X_t)\iprod{X_t}{\hat{\theta}_{i,t}}^2}+\gamma\EE{\sum_{t=1}^T\sum_{i\in V}\left(\frac{1}{K}-\pi_T(i|X_t)\right)\iprod{X_t}{\hat{\theta}_{i,t}}}.
\end{align*}
Furthermore, combining Claim~\ref{claim:2} with the fact that
\begin{equation*}
\EE{\sum_{t=1}^T\sum_{i\in V}\left(\frac{1}{K}-\pi_T(i|X_t)\right)\iprod{X_t}{\hat{\theta}_{i,t}}}=\EE{\sum_{t=1}^T\sum_{i\in V}\left(\frac{1}{K}-\pi_T(i|X_t)\right)\iprod{X_t}{\theta_{i,t}}}\le 2T,    
\end{equation*}
and $(1-\gamma)\le 1$, we conclude with
\begin{equation*}
    \EE{\sum_{t=1}^T\sum_{i\in V}(\pi_t^a(i|X_t)-\pi_T(i|X_t))\iprod{X_t}{\hat{\theta}_{i,t}}}\le \frac{\log{K}}{\eta}+2\gamma T+\eta\EE{\sum_{t=1}^T\sum_{i\in V}\pi_t^a(i|X_t)\iprod{X_t}{\hat{\theta}_{i,t}}^2}.
\end{equation*}
Since the above steps hold for any $\pi_T\in\Pi$, we have that
\begin{equation*}
    \mR_T\le \frac{\log{K}}{\eta}+2\gamma T+\eta\EE{\sum_{t=1}^T\sum_{i\in V}\pi_t^a(i|X_t)\iprod{X_t}{\hat{\theta}_{i,t}}^2}.
\end{equation*}
\end{proof}

\subsection{Proof of Theorem~\ref{thm:1}}

Before presenting the proof  of Theorem~\ref{thm:1}, we restate the following two graph-theoretic results from~\citet{mannor2011bandits,alon2017nonstochastic} for convenience.

\begin{lemma}[Lemma 10 in~\citet{alon2017nonstochastic}]
\label{lem:undirected}
Let $G_t$ be an undirected graph. For any distribution $\pi$ over $V$,
\begin{equation*}
    \sum_{i\in V}\frac{\pi(i)}{\pi(i)+\sum_{j:j\xrightarrow{t}i}\pi(j)}\le\alpha(G_t).
\end{equation*}
\end{lemma}
\begin{lemma}[Lemma 5 in~\citet{alon2015online}]
\label{lem:directed}
Let $G_t$ be a directed graph and $\pi$ be any probability distribution over $V$. Assume that $\pi(i)\ge \epsilon$ for all $i\in V$ for some constant $0<\epsilon< \frac{1}{2}$. Then,
\begin{equation*}
    \sum_{i\in V}\frac{\pi(i)}{\pi(i)+\sum_{j:j\xrightarrow{t}i}\pi(j)}\le 4\alpha(G_t)\log{\frac{4K}{\alpha(G_t)\epsilon}}.
\end{equation*}
\end{lemma}

\begin{proof}[Proof of Theorem~\ref{thm:1}]
\label{pf:thm1}

Using Lemma~\ref{lem:1}, we are left to upper bound the term $\EE{\sum_{t=1}^T\sum_{i\in V}\pi_t^a(i|X_t)\iprod{X_t}{\hat{\theta}_{i,t}}^2}$. Substituting Eq.~\eqref{eq:est1} into this term yields,
\begin{align}
    \EE{\sum_{i\in V}\pi_t^a(i|X_t)\iprod{X_t}{\hat{\theta}_{i,t}}^2}&=\EE{\sum_{i\in V}\pi_t^a(i|X_t)\frac{\indfunc\{i\in S_{I_t,t}\}\tilde{l}_t^2(\tilde{X}_t,i)}{q^2_t(i|X_t)}X_t^\top\Sigma^{-1}\tilde{X}_t\tilde{X}_t^\top\Sigma^{-1}X_t}\nonumber\\
    &\overset{(\text{a})}{\le}\EE{\sum_{i\in V}\pi_t^a(i|X_t)\frac{\indfunc\{i\in S_{I_t,t}\}}{q^2_t(i|X_t)}X_t^\top\Sigma^{-1}\tilde{X}_t\tilde{X}_t^\top\Sigma^{-1}X_t}\label{eq:oracle_exp}\\
    &\overset{(\text{b})}{=}\EE{\sum_{i\in V}\underbrace{\EEcct{\frac{\indfunc\{i\in S_{I_t,t}\}}{q^2_t(i|X_t)}}{X_t}}_{A}\pi_t^a(i|X_t)X_t^\top\Sigma^{-1}\tilde{X}_t\tilde{X}_t^\top\Sigma^{-1}X_t}\nonumber,
\end{align}
where the step (a) is due to the fact that $\tilde{l}_t^2(\tilde{X}_t,i)\le 1$, and step (b) uses the law of total expectation.
We have the following result for term $A$:
\begin{align*}
    A=\sum_{j:i\in S_{j,t}}\frac{\pi(j|X_t)}{q^2_t(i|X_t)} = \frac{q_t(i|X_t)}{q^2_t(i|X_t)}=\frac{1}{q_t(i|X_t)}.
\end{align*}
According to Lemmas~\ref{lem:undirected} and~\ref{lem:directed}, we know that
\begin{equation*}
\sum_{i\in V}\frac{\pi_t^a(i|X_t)}{q_t(i|X_t)}\le Q_t,
\end{equation*}
where $Q_t$ is $\alpha(G_t)$ for undirected graph setting and $4\alpha(G_t)\log(4K^2/(\alpha(G_t)\gamma))$ for directed graph setting. Also, $Q_t$ is independent of $X_t$ and $\tilde{X}_t$. Thus,
\begin{align*}
\EE{\sum_{i\in V}\pi_t^a(i|X_t)\iprod{X_t}{\hat{\theta}_{i,t}}^2}&\le \EE{Q_t}\EE{X_t^\top\Sigma^{-1}\tilde{X}_t\tilde{X}_t^\top\Sigma^{-1}X_t}\\
&=\EE{Q_t}\EE{\text{tr}(\Sigma^{-1}\tilde{X}_t\tilde{X}_t^\top\Sigma^{-1}X_tX_t^\top)}\\
&=d\EE{Q_t}.
\end{align*}
In addition, we must ensure that $\eta\left|\iprod{X_t}{\hat{\theta}_{i,t}}\right|\le 1$ for all $t=1,\ldots,T$:
\begin{equation*}
\left|\iprod{X_t}{\hat{\theta}_{i,t}}\right|=\frac{\indfunc\{i\in S_{i,t}\}}{q_t(i|X_t)}\left|X_t^\top\Sigma^{-1}\tilde{X}_t\tilde{l}_t(\tilde{X}_t,i)\right|\le\frac{K\sigma^2}{\lambda_{\text{min}}\gamma}, 
\end{equation*}
where we use the fact that $q_t(i|X_t)\ge\pi_t^a(i|X_t)\ge\frac{\gamma}{K}$, $|\tilde{l}_t(\tilde{X}_t,i)|\le 1$, and $\left|X_t^\top\Sigma^{-1}\tilde{X}_t\right|\le\frac{\sigma^2}{\lambda_{\text{min}}}$. Choosing $\gamma = \frac{\eta K\sigma^2}{\lambda_\text{min}}$ guarantees $\eta\left|\iprod{X_t}{\hat{\theta}_{i,t}}\right|\le 1$, which concludes the proof.
\end{proof}

\subsection{Proofs of Corollaries~\ref{col:1} and~\ref{col:2}}
\begin{proof}[Proof of Corollary~\ref{col:1}]
Given the fact $Q_t = \alpha(G_t)$ in Theorem~\ref{thm:1}, and assuming $\alpha(G_t)\le \alpha_t$ for $t = 1,\ldots,T$, we conclude that
\begin{equation*}
 \mR_T= \mathcal{O} \left(\sqrt{\left(2K\sigma^2T/\lambda_{\text{min}}+d\sum_{t=1}^T \alpha_t \right)\log{K}}\right),
\end{equation*}
by setting $\eta = \sqrt{\frac{\log{K}}{2K\sigma^2T/\lambda_{\text{min}}+d\sum_{t=1}^T \alpha_t}}$.
\end{proof}
\begin{proof}[Proof of Corollary~\ref{col:2}]
Define $f(z) = 4z\log(4K^2/(z\gamma))$ for $z\le K$, and we have that
\begin{equation*}
    f'(z) = 4\log{\frac{4K^2}{\gamma}}-4\log{z}-4.
\end{equation*}
Notice that $4\log(4K^2)> 4\log{z}$, and so $f(z)$ is an increasing function as long as $\log(1/\gamma)\ge 1$. If $\alpha(G_t)\le\alpha_t$ for $t=1,\ldots, T$, the following result holds if $\log(1/\gamma)\ge 1$,
\begin{equation*}
    \EE{4\alpha(G_t)\log{\frac{4K^2}{\alpha(G_t)\gamma}}}\le 4\alpha_t\log{\frac{4K^2}{\alpha_t\gamma}}.
\end{equation*}
By choosing $\eta = \sqrt{1/(K\sigma^2T/\lambda_{\text{min}}+4d\sum_{t=1}^T \alpha_t)}$ and $\gamma = \frac{\eta K\sigma^2}{\lambda_\text{min}}$, we conclude that
\begin{equation*}
    \mR_T=\mathcal{O}\left(\sqrt{\left(\frac{K\sigma^2}{\lambda_{\text{min}}}T+4d\sum_{t=1}^T \alpha_t\right)}\log(KdT)\right).
\end{equation*}
\end{proof}

\section{Proofs for Section~\ref{sec:ix}}
\label{apen:2}

\subsection{Proof of Claim~\ref{claim:3}}
\label{pf:claim3}
\begin{proof}
\begin{align*}
    \EEcct{\sum_{i\in V}\pi_t^a(i|X_t)\iprod{X_t}{\hat{\theta}_{i,t}}}{X_t}& = \sum_{i\in V}\pi_t^a(i|X_t)\frac{1}{q_t(i|X_t)+\beta_t}X_t^T\Sigma^{-1}\EEcct{\indfunc\{i\in S_{I_t,t}\}\tilde{X}_t\tilde{X}_t^\top}{X_t}\theta_{i,t}\\
    &\overset{(\text{a})}{=}\sum_{i\in V}\pi_t^a(i|X_t)\frac{q_t(i|X_t)}{q_t(i|X_t)+\beta_t}\iprod{X_t}{\theta_{i,t}}\\
    &=\sum_{i\in V} \pi_t^a(i|X_t)\iprod{X_t}{\theta_{i,t}}-\beta_t\sum_{i\in V}\frac{\pi_t^a(i|X_t)}{q_t(i|X_t)+\beta_t}\iprod{X_t}{\theta_{i,t}},
\end{align*}
where the equality (a) holds because $\indfunc\{i\in S_{I_t,t}\}$ and $\tilde{X}_t$ are independent.
\end{proof}

\subsection{Proof of Theorem~\ref{thm:2}}
\label{pf:thm2}
To prove Theorem~\ref{thm:2}, we need the graph-theoretic result from~\citet{kocak2014efficient}, which is restated here.

\begin{lemma}[Lemma 2 in~\citet{kocak2014efficient}]
\label{lem:IX}
Let $G_t$ be a directed or undirected graph with vertex set $V:=\{1,\ldots,K\}$. Let $\alpha(G_t)$ be the independence number of $G_t$ and $\pi$ be a distribution over $V$. Then,
\begin{equation*}
\sum_{i\in V}\frac{\pi(i)}{c+\pi(i)+\sum_{j:j\xrightarrow{t}i}\pi(j)}\le 2\alpha(G_t)\log\left(1+\frac{\lceil K^2/c\rceil+K}{\alpha(G_t)}\right)+2,
\end{equation*}
where $c$ is a positive constant.
\end{lemma}

\begin{proof}[Proof of Theorem~\ref{thm:2}]
We start by recalling the notation $w_t(x,i)=\exp\left(-\eta_t\sum_{s=1}^{t-1}\iprod{x}{\hat{\theta}_{i,s}}\right)/K$ in Eq.~\eqref{eq:est2}, and introducing $W_t(x) = \sum_{i\in V}w_t(x,i)$ and $W'_t(x)=\sum_{i\in V}\exp\left(-\eta_{t-1}\sum_{s=1}^{t-1}\iprod{x}{\hat{\theta}_{i,s}}\right)/K$. The proof follows~\citet{kocak2014efficient} with some additional techniques.
\begin{align}
    \frac{1}{\eta_t}\log\frac{W'_{t+1}(X_t)}{W_t(X_t)}&= \frac{1}{\eta_t}\log\left(\sum_{i\in V}\frac{w_t(X_t,i)}{W_t(X_t)}e^{-\eta_t\iprod{X_t}{\hat{\theta}_{i,t}}}\right)\nonumber\\
    & = \frac{1}{\eta_t}\log\left(\sum_{i\in V}\pi_t^a(i|X_t)e^{-\eta_t\iprod{X_t}{\hat{\theta}_{i,t}}}\right)\nonumber\\
    &\overset{(\text{a})}{\le}\frac{1}{\eta_t}\log\left(\sum_{i\in V}\pi_t^a(i|X_t)\left(1-\eta_t\iprod{X_t}{\hat{\theta}_{i,t}}+\frac{1}{2}\eta_t^2\iprod{X_t}{\hat{\theta}_{i,t}}^2\right)\right)\nonumber\\
    &=\frac{1}{\eta_t}\log\left(1+\sum_{i\in V}\pi_t^a(i|X_t)\left(-\eta_t\iprod{X_t}{\hat{\theta}_{i,t}}+\frac{1}{2}\eta_t^2\iprod{X_t}{\hat{\theta}_{i,t}}^2\right)\right)\nonumber\\
    &\overset{(\text{b})}{\le}\frac{1}{\eta_t}\sum_{i\in V}\pi_t^a(i|X_t)\left(-\eta_t\iprod{X_t}{\hat{\theta}_{i,t}}+\frac{1}{2}\eta_t^2\iprod{X_t}{\hat{\theta}_{i,t}}^2\right),\label{eq:thm2_1}
\end{align}
where step (a) uses the inequality $\exp(-z)\le 1-z+z^2/2$ that holds for $z\ge 0$ and step (b) uses the inequality $\log(1+z)\le z$ that holds for all $z>-1$. Notice that
\begin{align}
    W_{t+1}(X_t)&=\sum_{i\in V}\frac{1}{K}e^{-\eta_{t+1}\sum_{s=1}^t\iprod{X_t}{\hat{\theta}_{i,t}}}\nonumber\\
    &= \sum_{i\in V}\frac{1}{K}\left(e^{-\eta_t\sum_{s=1}^t\iprod{X_t}{\hat{\theta}_{i,t}}}\right)^{\frac{\eta_{t+1}}{\eta_t}}\nonumber\\
    &\overset{(\text{a})}{\le} \left(\frac{1}{K}\sum_{i\in V}e^{-\eta_t\sum_{s=1}^t\iprod{X_t}{\hat{\theta}_{i,t}}}\right)^{\frac{\eta_{t+1}}{\eta_t}}\nonumber\\
    &=(W'_{t+1}(X_t))^{\frac{\eta_{t+1}}{\eta_t}},\label{eq:thm2_2}
\end{align}
where step (a) uses Jensen's inequality for the concave function $z^{\frac{\eta_{t+1}}{\eta_t}}$ for all $z\in\R$ as $\eta_t$ is a decreasing sequence. Taking the $\log(\cdot)$ on both side of Eq.~\eqref{eq:thm2_2}, we have that
\begin{equation*}
 \frac{1}{\eta_t}\log\frac{W'_{t+1}(X_t)}{W_t(X_t)}\ge \frac{\log{W_{t+1}(X_t)}}{\eta_{t+1}}-\frac{\log W_t(X_t)}{\eta_t}.   
\end{equation*}
The following fact that can be easily interpreted using the same techniques as Lemma~\ref{lem:1}:
\begin{align}
    \EE{\sum_{t=1}^T\left(\frac{\log{W_{t+1}(X_t)}}{\eta_{t+1}}-\frac{\log{W_t(X_t)}}{\eta_t}\right)}&=\EE{\frac{\log{W_{T+1}(X)}}{\eta_{T+1}}-\frac{\log W_1(X)}{\eta_1}}\nonumber\\
    &\ge \EE{\frac{\log{w_{T+1}(X,\pi_T(X))}}{\eta_{T+1}}-\frac{\log W_1(X)}{\eta_1}}\nonumber\\
    &=-\EE{\frac{\log{K}}{\eta_{T+1}}}-\EE{\sum_{t=1}^T\iprod{X}{\hat{\theta}_{\pi_T(X),t}}}\nonumber\\
    &=-\EE{\frac{\log{K}}{\eta_{T+1}}}-\EE{\sum_{t=1}^T\iprod{X_t}{\hat{\theta}_{\pi_T(X_t),t}}},\label{eq:thm2_3}
\end{align}
where $X\sim D$ is independent from the whole interaction history $\mF_t$. Wrapping up above steps in Eq.~\eqref{eq:thm2_1} and Eq.~\eqref{eq:thm2_3}, and applying the Claim~\ref{claim:3} the $\theta_{i,t}$ is an optimistic estimator that $\EE{\sum_{t=1}^T\iprod{X_t}{\hat{\theta}_{\pi_T(X_t),t}}}\le \EE{\sum_{t=1}^T\iprod{X_t}{\theta_{\pi_T(X_t),t}}}$, we have that
\begin{equation}
\label{eq:thm2_4}
-\EE{\frac{\log{K}}{\eta_{T+1}}}-\EE{\sum_{t=1}^T\iprod{X_t}{\theta_{\pi_T(X_t),t}}}\le \EE{\sum_{i=1}^T\sum_{i\in V}\pi_t^a(i|X_t)\left(-\iprod{X_t}{\hat{\theta}_{i,t}}+\frac{1}{2}\eta_t\iprod{X_t}{\hat{\theta}_{i,t}}^2\right)}. 
\end{equation}
Notice that
\begin{align}
 \EE{\sum_{i\in V}\pi_t^a(i|X_t)\iprod{X_t}{\hat{\theta}_{i,t}}}&=\EE{\EEcct{\sum_{i\in V}\pi_t^a(i|X_t)\iprod{X_t}{\hat{\theta}_{i,t}}}{X_t}}\nonumber\\
    &\overset{(\text{a})}{=}\EE{\sum_{i\in V} \pi_t^a(i|X_t)\iprod{X_t}{\theta_{i,t}}-\beta_t\sum_{i\in V}\frac{\pi_t^a(i|X_t)}{q_t(i|X_t)+\beta_t}\iprod{X_t}{\theta_{i,t}}}\nonumber\\
    &\overset{(\text{b})}{\ge}\EE{\sum_{i\in V} \pi_t^a(i|X_t)\iprod{X_t}{\theta_{i,t}}-\beta_tQ_t},\label{eq:thm2_5}
\end{align}
where step (a) is due to Claim~\ref{claim:3}, and step (b) uses Lemma~\ref{lem:IX} and $\iprod{X_t}{\theta_{i,t}}\in[0,1]$. Also,
\begin{align}
&\EE{\eta_t\sum_{i\in V}\pi_t^a(i|X_t)\iprod{X_t}{\hat{\theta}_{i,t}}^2} = \EE{\EEcct{\eta_t\sum_{i\in V}\pi_t^a(i|X_t)\iprod{X_t}{\hat{\theta}_{i,t}}^2}{X_t}}\nonumber\\
&\qquad \le\EE{\EE{\eta_t}\sum_{i\in V}\frac{\pi_t^a(i|X_t)}{(q_t(i|X_t)+\beta_t)^2}X_t^\top\Sigma^{-1}\EEcct{\indfunc\{i\in S_{I_t,t}\}\tilde{X}_t\tilde{X}_t^\top}{X_t}\Sigma^{-1}X_t}\nonumber\\
&\qquad = \EE{\EE{\eta_t}\sum_{i\in V}\frac{\pi_t^a(i|X_t)q_t(i|X_t)}{(q_t(i|X_t)+\beta_t)^2}X_t^\top\Sigma^{-1}X_t}\nonumber\\
&\qquad = \EE{\eta_t\sum_{i\in V}\frac{\pi_t^a(i|X_t)q_t(i|X_t)}{(q_t(i|X_t)+\beta_t)q_t(i|X_t)}X_t^\top\Sigma^{-1}X_t}\nonumber\\
&\qquad\overset{(\text{a})}{\le}\EE{\eta_tQ_t\text{tr}(\Sigma^{-1}X_tX_t^\top)}\le \EE{\eta_tQ_t}d,\label{eq:thm2_6}
\end{align}
where step (a) uses Lemma~\ref{lem:IX}. By reordering the results in Eqs.~\eqref{eq:thm2_4},~\eqref{eq:thm2_5}, and~\eqref{eq:thm2_6}, we have that
\begin{equation}
\label{eq:thm2_7}
  \EE{\sum_{t=1}^T(\pi_t^a(i|X_t)-\pi_T(i|X_t))\iprod{X_t}{\theta_t}} \le \EE{\frac{\log{K}}{\eta_{t+1}}}+\sum_{t=1}^T\EE{\beta_tQ_t}+\frac{d}{2}\sum_{i=1}^T\EE{\eta_tQ_t}.
\end{equation}
Plugging in $\eta_t$ and $\beta_t$ and using Lemma 3.5 in~\citet{auer2002adaptive}, the result in Eq.~\eqref{eq:thm2_7} becomes
\begin{equation*}
    \mR_T  \le 2(1+\sqrt{d})\EE{\sqrt{\left(K+\sum_{t=1}^TQ_t\right)\log{K}}},
\end{equation*}
which holds for all $\pi_T\in \Pi$. 
\end{proof}

\subsection{Proof of Corollary~\ref{col:3}}
\begin{proof}
Notice that $x\log(1+a/x)$ is an increasing function of $ x \in (0,\infty]$ for any $a>0$, and thus 
\begin{equation*}
    Q_t\le 2\alpha_t\log\left(1+\frac{\lceil K^2/\beta_t\rceil+K}{\alpha_t}\right)+2,
\end{equation*}
if $\alpha(G_t)\le \alpha_t$ for $t=1,\ldots T$. Using the fact
\begin{equation*}
    \log\left(1+\frac{\lceil K^2/\beta_t\rceil+K}{\alpha_t}\right)\le \log\left(1+\frac{\lceil K^2\sqrt{tK/\log{K}}\rceil+K}{\alpha_t}\right)=\mathcal{O}(\log(KT)),
\end{equation*}
we conclude that
\begin{equation*}
    \mR_T=\mathcal{O}\left(\sqrt{\sum_{t=1}^T\alpha_td\log{K}\log{KT}}\right),
\end{equation*}
for both directed and undirected graph settings.
\end{proof}

\end{document}